\documentclass{article}
\usepackage{kotex}
\usepackage{amssymb, amsmath, amsthm}
\usepackage[pdftex]{graphicx}
\usepackage{booktabs}
\usepackage{multirow}
\usepackage{siunitx}
\usepackage[numbers]{natbib}
\usepackage{hyperref}
\usepackage{tabularx}
\usepackage[utf8]{inputenc}
\usepackage[left=2cm,right=2cm,top=2cm,bottom=2cm,a4paper]{geometry}
\usepackage{amsfonts}
\usepackage{amsmath}
\usepackage{amsthm}
\usepackage{booktabs} 
\usepackage[caption=false]{subfig}
\newtheorem{proposition}{Proposition}
\newtheorem{lemma}{Lemma}
\newtheorem{theorem}{Theorem}
\newtheorem{definition}{Definition}

\usepackage{authblk}

\begin{document}

\title{NEAR: Neighborhood Edge AggregatoR for Graph Classification}
\author[1]{\small Cheolhyeong Kim}
\author[2]{\small Haeseong Moon\thanks{Work done while at POSTECH}}
\author[1]{\small Hyung Ju Hwang\thanks{Corresponding Author}}

\affil[1]{Department of Mathematics, POSTECH}
\affil[2]{Department of Mathematics, University of California, San Diego}
\date{}

\maketitle
\begin{abstract}
Learning graph-structured data with graph neural networks (GNNs) has been recently emerging as an important field because of its wide applicability in bioinformatics, chemoinformatics, social network analysis and data mining. Recent GNN algorithms are based on neural message passing, which enables GNNs to integrate local structures and node features recursively. However, past GNN algorithms based on 1-hop neighborhood neural message passing are exposed to a risk of loss of information on local structures and relationships. In this paper, we propose Neighborhood Edge AggregatoR (NEAR), a framework that aggregates relations between the nodes in the neighborhood via edges. NEAR, which can be orthogonally combined with Graph Isomorphism Network (GIN), gives integrated information that describes which nodes in the neighborhood are connected. Therefore, NEAR can reflect additional information of a local structure of each node beyond the nodes themselves in 1-hop neighborhood. Experimental results on multiple graph classification tasks show that our algorithm makes a good improvement over other existing 1-hop based GNN-based algorithms.
\end{abstract}

\section{Introduction}
Interest in learning graph structured data has risen rapidly in recent years because of its wide applicability in bioinformatics, chemoinformatics, social network analysis and data mining. For learning graph-structured data, we need an algorithm that can effectively represent the graph structure and relations between the graph nodes. In recent years, numerous approaches to learn graph structure were developed, including graph kernel methods \citep{retgk,awl,shervashidze2009efficient,shervashidze2011weisfeiler,borgwardt2005shortest} and "neural message passing"  \citep{gilmer2017neural} based graph neural network (GNN) methods  \citep{diffpool,graphsage,dgcnn,patchysan,kipf2016semi,defferrard2016convolutional,DBLP:conf/icml/XuLTSKJ18,xie2018crystal}. 
    
Most of GNN algorithms aggregate feature information on connected nodes recursively, and thereby create new feature vectors for each node in the graph. By repeating this process, an algorithm gets information about $k$-hop neighborhood of each node and a representation of the whole graph by combining those feature vectors. \cite{xu2018how} formulated this GNN encompassing process mathematically by using a concept of multiset functions.

A major limitation of 1-hop based GNN algorithms is that each node uses only its neighborhood nodes' information, which does not comprise relationships of neighborhood nodes. This limitation causes GNN architectures map different neighborhoods into the same representation, which leads difficulties in learning the graph structure. To illustrate this limitation, we propose a family of artificial graphs that are impossible to classify using the 1-dimensional Weisfeiler-Lehman (WL) test \citep{weisfeiler1968reduction} and Graph Isomorphism Network (GIN); this shows that reflecting relationships between the nodes in the neighborhood is necessary.
    
To overcome this limitation, we propose Neighborhood Edge AggregatoR (NEAR), a simple framework that aggregates edges in a neighborhood of a given node. Our idea was inspired by noticing that certain graph structures that cannot be classified by previous 1-hop based message-passing GNN frameworks. Incorporating NEAR and graph isomorphism network (GIN) \citep{xu2018how} framework, we can combine relationships between the nodes in the neighborhood and the existing node feature vectors. Our proposed algorithm enables the graph representation to reflect the relationships between the nodes in the neighborhood.

Our main contributions can be summarized as follows.

\begin{itemize}
\item
We constructed a family of graphs that cannot be classified by the existing GNN models that are based on 1-hop neighborhood aggregator, thus claiming that reflecting relationships between the nodes in neighborhoods is required to represent their local structure.
\item
We proposed NEAR, a simple framework that aggregates the local structure of neighborhoods which can be used with GIN. We verified that our framework can reflect the local structures of graphs well enough. Performance of GIN has been improved with NEAR.
\item
We proposed simple variants of NEAR, which have a more powerful discriminative power to classify local structures. Our variants of NEAR showed compatible performances on various graph classification tasks \citep{KKMMN2016} among GNN baseline algorithms.
\end{itemize}

\section{Related Works}
\textbf{Graph Neural Network:} Graph neural networks (GNNs) \citep{gori2005, DBLP:journals/tnn/ScarselliGTHM09}, especially emphasizing graph convolutional networks (GCNs) \citep{defferrard2016convolutional, kipf2016semi}, have been widely studied due to their successful results \citep{kipf2016semi, graphsage} in node classification, link prediction, and graph classification tasks. Motivated from these studies, advanced techniques in GCN such as skip-connection \citep{DBLP:conf/icml/XuLTSKJ18}, attention in graph \citep{DBLP:conf/cvpr/MontiBMRSB17, velickovic2018graph}, capsule-GNN \citep{DBLP:conf/iclr/XinyiC19}, graph pooling \citep{diffpool, DBLP:conf/icml/LeeLK19, DBLP:conf/aaai/RanjanST20}, graph generation \citep{molgan, DBLP:conf/nips/YouLYPL18}, graph auto-encoder \citep{DBLP:journals/corr/KipfW16a} are also rapidly emerging and discussed. Applications on drug discovery with drug-drug interaction \citep{PARK2020113538, DBLP:conf/bcb/KarimCJUBD19}, bioinformatics \citep{DBLP:conf/nips/FoutBSB17} and chemoinformatics \citep{C7SC02664A, C8SC04228D, Kojima2020}, and knowledge graphs \citep{DBLP:conf/nips/HamiltonBZJL18, DBLP:conf/nips/Kazemi018, DBLP:conf/www/WangZXLG19} are also widely studied, where the data can be represented in a graph form that includes an interaction between two nodes.\\
\textbf{Weisfeiler-Lehman test:} The evolution of GNN is influential in solving problems of graph theory and its related algorithms, and ideas in graph theory also improve GNN's performance as well. One of the focused algorithms in graph theory which recently arises in GNN literature is a graph isomorphism test. Weisfeiler-Lehman test (WL-test), in general k-dimensional WL test, determines that given two graphs are isomorphic or not. WL test has been firstly proposed by \cite{weisfeiler1968reduction} and widely used nowadays in a form of Weisfeiler-Lehman subtree kernel \citep{wltesting, DBLP:journals/jmlr/ShervashidzeSLMB11}, whereas a family of corner-case counterexamples was discovered by \cite{DBLP:journals/combinatorica/CaiFI92}. In terms of encoding a graph into a vector, the question of whether GNN can distinguish two non-isomorphic graphs has been naturally arisen. \\
\textbf{Relationship between 1-dim WL test and GNN:} Several studies have been proposed to connecting WL test and GNN in recent GNN researches. In Graph isomorphism network (GIN) \citep{xu2018how}, authors formulated their message-passing GNN with multiset functions and compared with 1-dim WL test. Their results claim that the choice of aggregator functions on the nodes in 1-hop neighborhood effects on the capacity of the GNN model, and prove that GIN has an equal performance with 1-dim WL test. Also, there are several attempts to reflect WL test with a higher dimension to GNN algorithm, such as k-GNN \citep{DBLP:conf/aaai/0001RFHLRG19} and disentangled GCN \citep{DBLP:conf/icml/Ma0KW019}.

\section{Proposed method}
\subsection{Preliminaries}
\label{sec2.1:prelim}
GNN's neighborhood aggregator and graph-level readout function operate on a set of feature vectors of nodes, potentially admitting the same feature vectors \citep{xu2018how}. Therefore, we first introduce a generalized concept of sets that allows repetition of elements.

\begin{definition} (Multiset)
A multiset $\mathcal{X}$ is a generalized concept of sets that allows repetition of elements. Multiset $\mathcal{X}$ can be represented as a pair of a set $S \subset \mathbb{R}^n$ and a function $m : S \rightarrow \mathbb{N}$, namely $\mathcal{X}=(S, m)$. $S$ represents a set of unique elements in $\mathcal{X}$ and $m$ represents multiplicities of each element in $S$.
\end{definition}

For example, two multisets $\{a,a,b,b,c\}$ and $\{b,c,a,a,b\}$ can be represented as $(\{a,b,c\}, m)$ with $m(a)=m(b)=2, m(c)=1$. We can easily observe that a multiset is invariant under permutation, because its underlying set and multiplicities remain the same under permutation. Therefore, functions that operate on multisets should be at least permutation invariant to be well-defined. Typical examples of multiset function are count (for finite case), summation (sum), average (mean), and min/max.

According to \cite{xu2018how}, the main structure of message-passing based GNN layer can be formulated using three core functions: AGGREGATOR, COMBINE, and READOUT. Given a graph $G = (V, E)$, suppose that there exist feature vectors on set of nodes with $H=\{h_v|v\in V\}$. Let $h_v$ be a feature vector of node $v\in V$, $h_{N_v}$ be an aggregated feature vector of neighborhood $N_v$ of node $v$, and $h_G$ be a representation vector of graph $G$. In this case, $h_{N_v}, h_G$ and a set of new feature vectors $\Phi(H) = \{h_v^{(new)} | v \in V\}$ can be defined as below.

\begin{align}
h_{N_v} &= AGGREGATOR(\{h_u|u\in N_v\}) \\
h_v^{(new)} &= COMBINE(h_v, h_{N_v}) \\
h_G &= READOUT(\{h_v | v\in V\})
\end{align}

First, AGGREGATOR operates on a set of feature vectors of neighborhood $N_v$ of node $v$. AGGREGATOR integrates information of $N_v$ and returns a feature vector $h_{N_v}$ that represents neighborhood of $v$. Second, COMBINE operates on $h_v$ and $h_{N_v}$ and creates a new feature vector of node in the next GNN layer. While repeating this for every GNN layer, READOUT function operates on the set of nodes in $G$ and returns a vector that represents the whole graph $G$. Summation was used as AGGREGATOR and summation/mean were used as READOUT \citep{xu2018how}. 2-layer MLP with learnable parameters was used as COMBINE to approximate an injective universal multiset function.

Let $H^{(k)} = \{h_v^{(k)} | v \in V\}$ be a multiset of feature vectors of nodes in $k{th}$ GNN layer $\Phi^{(k)}$, where $H^{(0)} = \{h_v^{(0)} | v \in V\}$ be a multiset of the given initial feature vectors of nodes. Because GNN layers are stacked in a row, every layer computes its new feature vectors recursively: $H^{(k)} = \Phi^{(k)}(H^{(k-1)})$ for $k \geq 1$. By stacking $k$ GNN layers in a row, we can expect that the model can learn up to $k$-hop neighborhood's representation.

\subsection{Toy example}
\label{sec2.2:toyex}

Summing AGGREGATOR in GIN aggregates the information on neighborhood's size and distribution \citep{xu2018how}. However, besides GIN, the current GNNs only with any simple 1-hop AGGREGATOR (regardless of sum/mean) may misclassify sets with different local structures. 

Here, we introduce a family of graphs that cannot be distinguished by GIN. Figure \ref{toyex} is an example of graphs that have the same neighborhood set with different local structures. 

\begin{lemma}
\label{lemma0}
There exists a graph $G=(V,E)$ with a multiset of feature vectors of nodes $H = \{h_v | v\in V\} = (\{h_w, h_b\}, m_N)$ with $m_N(h_b)=m_N(h_w)=2N$, satisfying the following conditions for every $N \in \mathbb{N}$.
\begin{itemize}
\item $G = (V,E)$ where $|V| = 4N$ and $|E| = 5N$. 
\item There are $2N$ black nodes and $2N$ white nodes.
\item Every white node is connected with two black nodes and has the same feature vector $h_w$.
\item Every black node is connected with two white nodes and one black node, and has the same feature vector $h_b$.
\end{itemize}
\end{lemma}

\begin{proof}
Let $V_B = \{1,2,\cdots,2N-1,2N\}$ be a set of black nodes and $V_W = \{2N+1,2N+2,\cdots,4N\}$ be a set of white nodes, where two black nodes $2k-1$ and $2k$ are connected for $1\leq k \leq N$. We define \textit{a multiset of black nodes \textbf{with multiplicities 2}} as $\mathcal{B} = \{1,1,2,2,3,3,\cdots,2N,2N\}$. 

It is enough to show that the multiset $\mathcal{B}$ can be \textbf{partitioned} into $2N$ sets $B_{2N+1}, B_{2N+2}, ..., B_{4N}$, each set $B_j$ comprises \textbf{two distinct black node elements}. If such partition exists, then the only remaining part is connecting each white node $j$ to two black nodes $j_1, j_2$ of $B_j = \{j_1, j_2\}$. 

It is clear that each white node $j$ is only connected with two black nodes $\{j_1, j_2\}$ above. In the perspective of a black node $i$, every $i$ is contained in two partitioned sets $B_{i_1}, B_{i_2}$ with $i_1 \neq i_2$, since $\mathcal{B} =  \{1,1,2,2,3,3,\cdots,2N,2N\}$ is partitioned into $2N$ sets $B_{2N+1}, ..., B_{4N}$. Thus, each black node $i$ is connected with two white nodes $i_1, i_2$ and one black node $i^*$, where $i^* = 2k$ if $i = 2k-1$ and $i^* = 2k-1$ if $i=2k$. \\

For $2N+1 \leq j \leq 4N-2$, if $|\mathcal{B}| > 4$, we randomly pick 2 different elements $j_1, j_2$ in $\mathcal{B}$ and define $B_j = \{j_1, j_2\}$. After then, we remove $j_1, j_2$ from the multiset $\mathcal{B}$. Next, we connect them with the given white node $j$. Repeating this procedure, we get 2 remaining white nodes $\{4N-1, 4N\}$ and 4 black node elements in $\mathcal{B}$.

If $|\mathcal{B}| = 4$, then we have three possible cases. 

\begin{itemize}
    \item Without loss of generality, if $\mathcal{B} = \{p,p,q,q\}$, then we pick $\{p,q\}$ and connect them with the white node $4N-1$. For the white node $4N$, we connect it to black nodes $p, q$.
    \item Without loss of generality, if $\mathcal{B} = \{p,p,q,r\}$, then we pick $\{p,q\}$ and connect them with the white node $4N-1$. For the white node $4N$, we connect it to black nodes $p, r$.
    \item If all elements in $\mathcal{B}$ are distinct, then we choose 2 elements randomly and connect them to the white node $4N-1$. The remaining elements will be connected with the white node $4N$.
\end{itemize}

Finishing the procedures above, every white node is connected with two black nodes and every black node is connected with two white nodes and one black node, which generates the desired graph. This completes the proof.
\end{proof}

\begin{figure}[htb]
  \centering
  \includegraphics[width=0.85\linewidth]{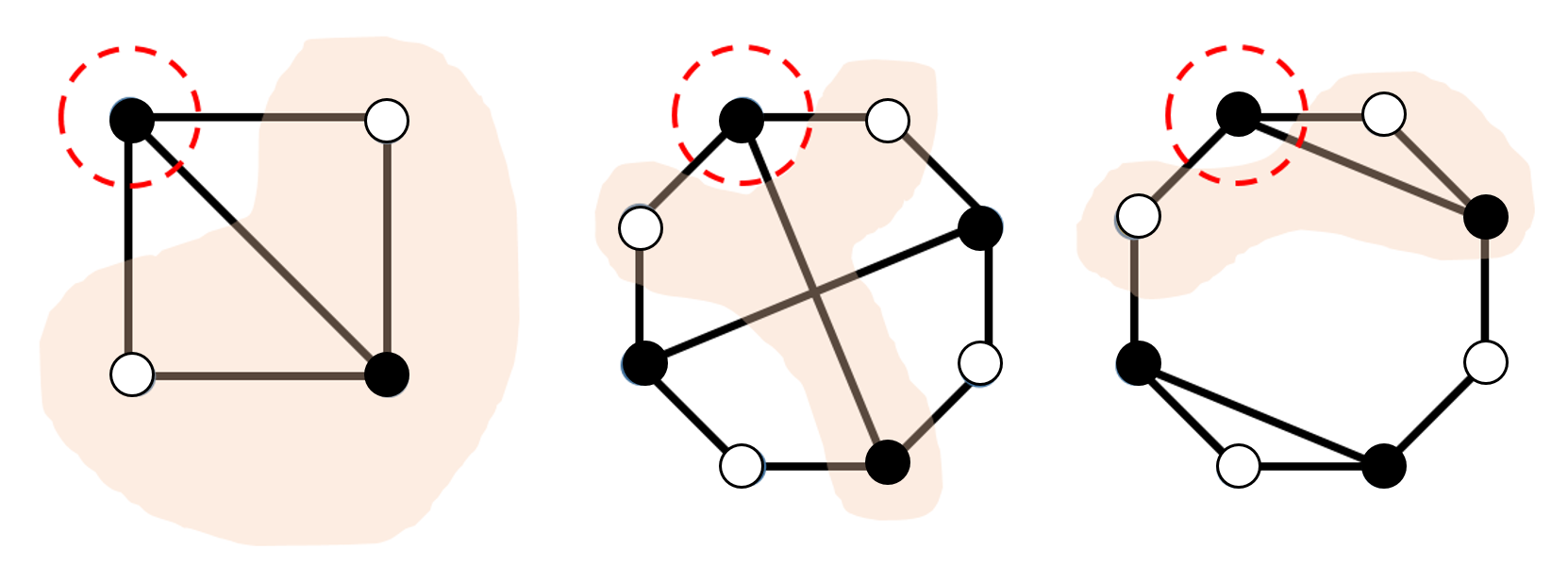}
  \caption{Examples that satisfy the conditions in lemma \ref{lemma0}. Every black node is connected with two white nodes and one black node. Thus, GNN layer will map the same vector on every black node in the next layer. However, in the shaded area, all black nodes have different local structures.}
  \label{toyex}
\end{figure}

Suppose black nodes have its feature vector $h_b$ and white nodes have its feature vector $h_w$. GIN aggregates neighborhood of target node $v$ and combine it with feature vector of $v$. These only take into account the combination of feature vectors of the nodes in $N_v$, not the relationship between the nodes in $N_v$. Therefore, although marked black nodes in Fig. \ref{toyex} have different local structures, they will be mapped to the same feature vector in the next GNN layer. The following proposition is a general statement for graphs in Fig. \ref{toyex}.

\begin{proposition}
\label{lem1}
Let $G$ be a graph that satisfies the conditions in lemma \ref{lemma0} with its feature vector multiset $H$. For a GNN layer $\Phi$ with 1-hop neighborhood AGGREGATOR and COMBINE, $G$ with a new feature vector set $\Phi(H)$ also satisfies the conditions in lemma \ref{lemma0}.
\end{proposition}

\begin{proof}
1-hop neighborhood AGGREGATOR and COMBINE do not modify the graph structure. Therefore, it is sufficient to prove that all nodes with degree 2 and degree 3 have the same feature vector mapped by the GNN layer respectively. Let $h_w$ be a feature vector of the white node and $h_b$ be a feature vector of the black node. Then, a new feature vector for the white node can be represented as below.

\begin{align}
&h_w^{(new)} = COMBINE(h_w, AGGREGATOR(\{h_b,h_b\}))
\end{align}

Because every white node has the neighborhood with the same feature vectors $\{h_b,h_b\}$, this can be applied to all white nodes in $G$, regardless of $COMBINE$ function or $AGGREGATOR$ function. Therefore, all white nodes are mapped to the same feature vector $h_w^{(new)}$. For black nodes, this can be similarly proved from the fact

\begin{align}
&h_b^{(new)} = COMBINE(h_b, AGGREGATOR(\{h_b,h_w,h_w\}))
\end{align}

\end{proof}

Every black node in such graph $G$ has the same feature vector, degree, and multiset of feature vectors of neighborhood nodes. Therefore, every black node will be mapped to the same feature vector in the next GNN layer. Likewise, every white node will be mapped onto the same feature vector. Repeating these procedures for every GNN layer with our proposed family of graphs, we obtain the following theorem which states that GIN may fail to catch differences in the local structures.

\begin{theorem}
For a family $\mathcal{G}$ of graphs introduced in lemma \ref{lemma0}, suppose a GNN model contains $K$ GNN layers, which contain 1-hop neighborhood AGGREGATOR and COMBINE. Let a vector $h_G^{(k)} = READOUT(\{h_v^{(k)}|v\in V\})$ be a representation vector of graph $G$ in $k^{th}$ GNN layer. Define a representation vector of graph $h^{(rep)}_G$ as $h^{(rep)}_G = CONCAT(h_G^{(0)},h_G^{(1)},...,h_G^{(K)})$, then
\begin{itemize}
    \item If mean is used as a READOUT, then GNN model maps every graph $G \in \mathcal{G}$ to the same vector, regardless of AGGREGATOR and COMBINE, i.e. $h^{(rep)}_G$ is constant for every graph $G \in \mathcal{G}$.
    \item In general, for any READOUT function, $h^{(rep)}_G$ is only dependent on $|V|$, where $|V|$ is the number of nodes in a graph $G$.
\end{itemize}
\label{thm1}
\end{theorem}

\begin{proof}
Let $G^{(0)}$ be an original graph $G$ with its node feature vectors and $G^{(k)}$ be a graph $G$ that passed $k$ GNN layers. From the Proposition \ref{lem1}, we can deduce that $G^{(k)}$ also satisfies the conditions in lemma \ref{lemma0} inductively. Therefore, there are $2N$ black nodes with degree 3 and $2N$ white nodes with degree 2 which have the same feature vector respectively in $k^{th}$ graph $G^{(k)}$. Let $h_b^{(k)}$, $h_w^{(k)}$ be the feature vectors of black/white nodes in $G^{(k)}$ respectively. Then,
\begin{align}
h_G^{(k)} &= READOUT(\{\underbrace{h_b^{(k)},\cdots,h_b^{(k)}}_\text{2N times},\underbrace{h_w^{(k)},\cdots,h_w^{(k)}}_\text{2N times}\}) \\ &= \Psi(h_b^{(k)},h_w^{(k)},2N)
\end{align}
Note that
\begin{align}
h_b^{(k)}, h_w^{(k)} &= \Phi^{(k)}(h_b^{(k-1)},h_w^{(k-1)}) = \cdots \\
&= \Phi^{(k)} \circ \cdots \circ \Phi^{(1)} (h_b^{(0)},h_w^{(0)})
\end{align}

Let $F^{(k)}=\Phi^{(k)} \circ \cdots \circ \Phi^{(1)}$, then $h_G^{(rep)}$ can be represented as below.

\begin{align}
    h_G^{(rep)} &= [h_G^{(0)},h_G^{(1)},\cdots,h_G^{(K)}] \\ &= [\Psi(h_b^{(0)},h_w^{(0)},2N),\cdots,\Psi(h_b^{(K)},h_w^{(K)},2N)] \\
    &= [\Psi(h_b^{(0)},h_w^{(0)},2N),\cdots,\Psi(F^{(K)}(h_b^{(0)},h_w^{(0)}),2N)] \\
    &= \mathcal{F}(h_b^{(0)},h_w^{(0)},2N)
\end{align}
Note that $h_b^{(0)},h_w^{(0)}$ are fixed for all graphs $G\in\mathcal{G}$. If READOUT function $\Psi$ is independent of the size of set $|V|=4N$, then $\Psi(h_b^{(k)},h_w^{(k)},2N) = \Psi(h_b^{(k)},h_w^{(k)})$ and $h_G^{(rep)}$ is also independent with $2N$. This immediately proves the first statement, because 

\begin{align}
MEAN(\{\underbrace{h_b^{(k)},\cdots,h_b^{(k)}}_\text{2N times},\underbrace{h_w^{(k)},\cdots,h_w^{(k)}}_\text{2N times}\}) = \frac{h_b^{(k)} + h_w^{(k)}}{2}
\end{align} 
is independent of $N$. In general, $h_G^{(rep)}$ is only dependent on $N = \frac{1}{4}|V|$, which directly proves the second statement. 
\end{proof}

Theorem \ref{thm1} illustrates that previous GIN model may miss valuable information on local structures in such cases illustrated in Fig.  \ref{toyex}. Not only the corner cases, this leads several 1-hop based GNN models to have lower model capacities. Therefore, the problem is to find a new graph neural network framework that can distinguish such differences.

\subsection{Neighborhood Edge Aggregator}
\label{sec2.3:near}
Here, we propose NEAR, a new GNN framework that aggregates information of a neighborhood via edges in a 1-hop neighborhood $N_v$ of a node $v$. We aim to obtain local structural properties by inserting an additional neighborhood edge aggregator. While aggregating and feed-forwarding the node's feature vector in every GNN layer, NEAR encodes the edges in $N_v$ and passes the information to the next layer. The definition below states an edge-aggregating process in our proposed algorithm NEAR in an abstract form.
\begin{definition}
Let $G = (V, E)$ be a graph with a multiset of node feature vectors $H = \{h_v |v\in V\}\subset\mathbb{R}^n$. Let $N_v$ be a set of nodes in neighborhood of $v\in V$ and $E_{N_v}$ be a set of edges that connect nodes in $N_v$. Suppose that $g : \mathbb{R}^{n} \times \mathbb{R}^{n} \rightarrow \mathbb{R}^{c}$ is a real-valued function, where $n$ is the dimension of feature vector of nodes and $c$ is the dimension of the embedded vectors. Let $\phi$ be a fixed multiset function. $NEAR_{g,\phi}$, which operates on graph $G$, node $v$, and feature vector multiset $H$, is defined as below.
\begin{align}
NEAR_{g,\phi}(G, v, H) = \phi(\{g(h_u,h_z) | u, z, \in N_v, (u,z) \in E_{N_v}\})
\end{align}
\end{definition}
We set $\phi$ as a summation in our work, which was motivated from the summation AGGREGATOR in \cite{xu2018how}. Then, $NEAR_{g,sum}(G, v, H)$ can be rewritten as below, where $e_{uz}$ is an adjacency matrix's element. In order to simplify the notation, we will state $NEAR_{g,sum}(G, v, H)$ as $NEAR_g(N_v, H)$.
\begin{align}
NEAR_{g,sum}(N_v, H) &= \sum_{uz \in E_{N_v}} g(h_u,h_z) \\&= \sum_{u, z \in N_v} e_{uz} g(h_u,h_z)
\end{align}

For a given node $v$, we add the feature vector $h_v^{(k)}$ and the aggregated neighborhood feature vector $h_{N_v}^{(k)}$, where $h_{N_v}^{(k)}$ is calculated by 1-hop neighborhood AGGREGATOR. In NEAR, edges in $N_v$ are additionally aggregated and mapped to $h_{NE_v}^{(k)}$, which is shown with bold edges. After then, two vectors $h_{v}^{(k)}+h_{N_v}^{(k)}$ and $h_{NE_v}^{(k)}$ are concatenated and mapped to a new feature vector of $v$ in the next GNN layer by COMBINE (MLP in Fig. \ref{fig:near}). These procedures will be done for every node in graph $G$. 

\begin{figure}[htb!]
    \centering
    \includegraphics[width=0.9\linewidth]{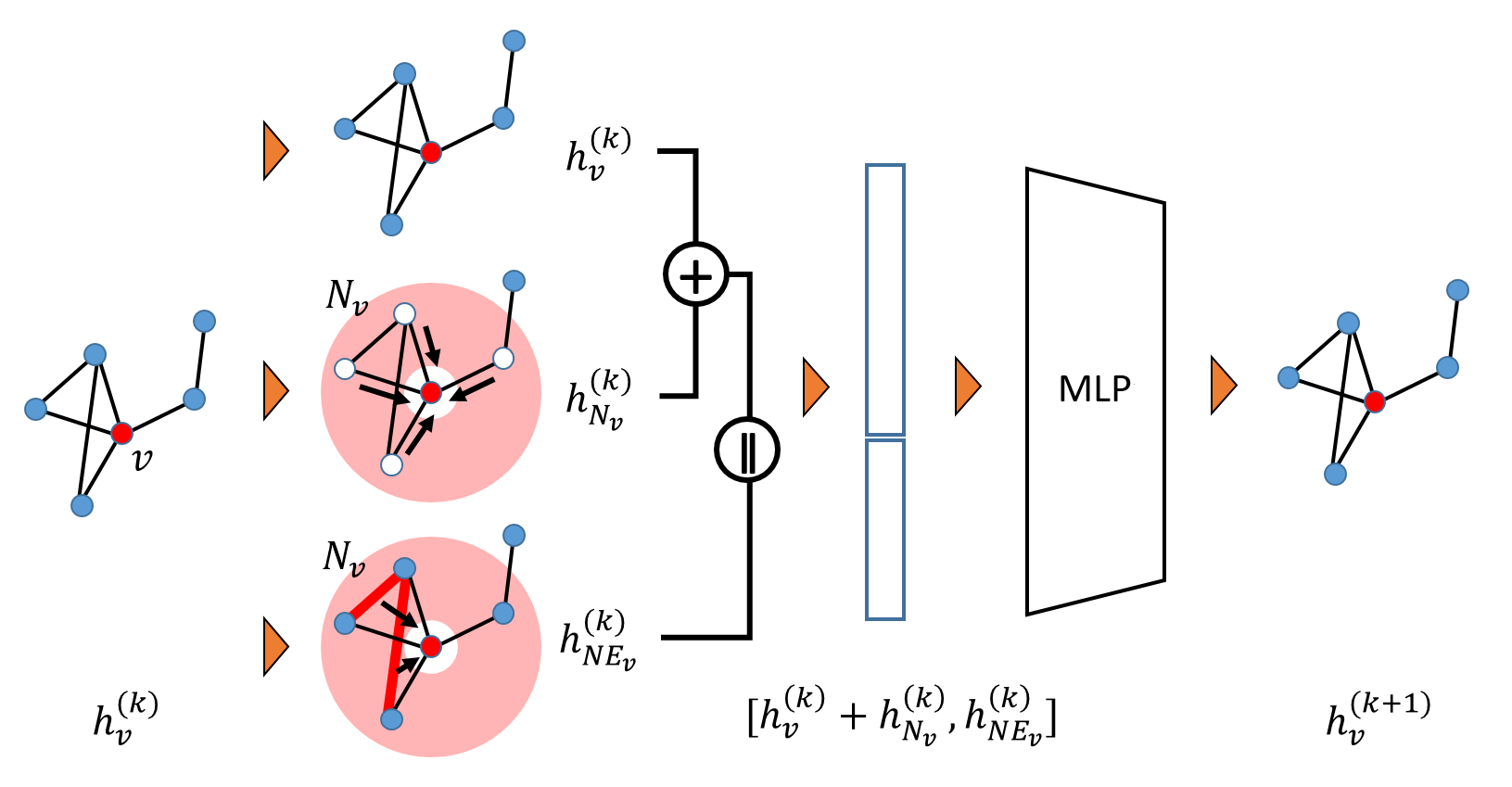}
    \caption{Brief sketch of GNN layer with GIN-0 and NEAR, which is also used in our experiments. GIN-0 is illustrated with $h_v^{(k)}$ and $h_{N_v}^{(k)}$ and NEAR is illustrated at the bottom with $h_{NE_v}^{(k)}$.}
    \label{fig:near}
\end{figure}

Note that $g(h_z, h_u)$ is sufficient to encode the connection between two nodes $z, w$ whose feature vectors are $h_z, h_u$. Once we define NEAR that maps the connection between the nodes in neighborhoods onto some hidden vector with size $c$, we can re-design the previous GIN architecture using our proposed method. NEAR can encode the local structure into every GNN layer.
\begin{align}
h^{(k)}_{N_v} &= AGGREGATOR(\{h_u^{(k)}|u\in N_v\}) \\
h^{(k)}_{NE_v} &= NEAR_g(N_v, H^{(k)}) \nonumber \\&= \sum_{u, z\in N_v} e_{uz} g(h^{(k)}_u,h^{(k)}_z) \\
h_v^{(k+1)} &= COMBINE(h_v^{(k)}, h_{N_v}^{(k)}, h_{NE_v}^{(k)}) \notag \\ &= MLP^{(k)}(CONCAT(h_v^{(k)} + h_{N_v}^{(k)}, h_{NE_v}^{(k)})) \\
h_G^{(k)} &= READOUT(\{h^{(k)}_v | v\in G\}) \\
h_G^{(rep)} &= CONCAT(\{ h_G^{(k)} | 0\leq k \leq K\})
\end{align}

\subsection{Proposed Variants of NEAR, Computational Cost, and Related Discussions}
We propose four simple variants of NEAR: NEAR-c, NEAR-e, NEAR-m, and NEAR-h. NEAR-c uses the simplest constant function $g_c(h_i, h_j) = 1$ and NEAR-e uses a simple addition function $g_e(h_i, h_j) = h_i + h_j$. NEAR-m uses an element-wise max function and NEAR-h uses the  Hadamard product. However, aggregating neighborhood edge information with a naive summation may cause heavy computational cost on large and dense graphs, while evaluating the term $\sum_{uz \in E_{N_v}} g(h_u,h_z) = \sum_{u, z \in N_v} e_{uz} g(h_u,h_z)$ over $E_{N_v}$ or $N_v \times N_v$. \\

For NEAR-c and NEAR-e, we can simplify our computation as below by incorporating and reformulating our NEAR term with local graph invariants. $d_i|_{N_v}$ is the number of nodes in $N_v$ that are connected with node $i$, which is equal to the number of triangles that contain node $i$ and node $j$.

\begin{align}
NEAR_c(N_v, H) &= \sum_{E_{N_v}} e_{ij} = |E_{N_v}| \\
NEAR_e(N_v, H) &= \sum_{E_{N_v}} h_i + h_j = \sum_{i\in N_v} d_i|_{N_v} h_i
\end{align}

These reformulations of NEAR-c and NEAR-e significantly reduce the time-complexity of calculating the NEAR term. That is, NEAR-c and NEAR-e compute  its summation $h_{NE_v}$ over $N_v$ (a set of neighborhood nodes of $v$) with equation (23) and (24), which can be computed simultaneously with $h_{N_v}$. These make its computational cost almost the same as GIN. Empirically, under the same model settings in our experiments, GIN, GIN with NEAR-c, and GIN with NEAR-e take a similar time per epoch, whereas NEAR-m and NEAR-h take 10x $\sim$ 20x times per epoch.\\

\subsubsection{Possible Future Works} We mainly focused on a particular form of NEAR with a summation multiset function and four simple functions $g$ that represent connections in a neighborhood, namely $NEAR_{g,sum}=NEAR-c/e/m/h$. Other variations of NEAR with different edge aggregator multiset function $\phi$ and edge feature map $g$ can be discovered in future works. Also, their reformulation with graph invariants can be further discussed with its computational cost reduction issue. One possible solution to handle the equation (19) is to approximate the neighborhood edge aggregation term by sampling. As an example, a sampling-based method as in GraphSAGE \citep{graphsage} can be considered, which will also be an interesting future work that can improve our proposed method with GIN.

\subsection{How powerful is NEAR?: Its limitation and a comparison with 3-WL test} It is known that GIN is as powerful as 1-WL test with an injective neighborhood node aggregator and an injective graph-level readout function. \citep{xu2018how} Clearly, 1-WL test and GIN cannot distinguish the proposed family of graphs as in Fig. \ref{toyex}. By collecting the information of edges $e\in E_{N_v}$ in a local neighborhood $N_v$, at least our proposed algorithm NEAR can distinguish the examples in Fig. \ref{toyex}. These graphs empirically show that NEAR is strictly more powerful than GIN, and surely increases the model capacity of GIN by incorporating the information of local neighborhood edges, equivalently the information of triangles.

From the perspective of a graph isomorphism as a benchmark of discriminative power of the model, two questions on the discriminative power of NEAR naturally arise.
\begin{itemize}
\item Is there an example that NEAR cannot distinguish two non-isomorphic graphs?
\item How powerful is NEAR aggregator compared with a higher-order WL test?
\end{itemize}

An encoded hidden vector passed through GIN with NEAR as in Fig. \ref{fig:near} can be rewritten as below:
\begin{align}
h_v^{(k+1)} &= COMBINE(h_v^{(k)},h_{N_v}^{(k)},h_{NE_v}^{(k)}) = MLP^{(k)}(h_v^{(k)}+h_{N_v}^{(k)},h_{NE_v}^{(k)}) \\
&= MLP^{(k)}(h_v^{(k)}+h_{N_v}^{(k)},\phi(\{g(h_u,h_z) | u, z, v \text{ are connected}\})
\end{align}

Note that GIN with NEAR aggregates the information of triangles that contain $v$ to $h_v^{(k+1)}$, from the equation (26). Therefore, we can consider the simplest counterexample: two non-isomorphic graphs that are not distinguishable with 1-WL test and do not contain triangles. They cannot be distinguished by NEAR since any informative information does not appear on a NEAR term $h_{NE_v}=0$. Also, from the fact that 3-WL test is executed over 3-tuples of vertices, we can deduce that if two graphs are not distinguishable with 3-WL test, then NEAR also cannot distinguish them. That is, GIN with NEAR is strictly more powerful than 1-WL test, but less powerful than 3-WL test.

\section{Experiments}
\label{sec3:exp}
We conduct two experiments to show the importance of relations between the neighborhoods and obtain a compatible performance over the existing GIN. In the experiments, we constructed GIN-0 based models with our variants of NEAR. Firstly, we performed two graph classification tasks with the toy examples in Section \ref{sec2.2:toyex}. Each task requires us to classify several graph properties. Secondly, we performed the graph classification for 9 benchmark datasets and our two toy example tasks. 10 times of 10-fold cross-validation was applied, and mean and standard deviation of the test accuracies were reported. Overall, our proposed model improved the previous results of GIN and other GNN-based models on several graph classification benchmarks. 

\subsection{Model Configuration of GIN}
\label{sec3.1:model}
Following the model configuration of \cite{xu2018how}, we impelemented GIN and NEAR variants based on the GIN-0. The GIN model in our experiments has 5 GNN layers stacked in a row; each of them has \textbf{sum} AGGREGATOR, \textbf{sum} READOUT and 2 fully-connected layers as COMBINE \citep{xu2018how}. COMBINE function of our baseline GIN-0 model is given by $h_v^{(k+1)} = MLP^{(k)}(CONCAT(h_v^{(k)} + h_{N_v}^{(k)}, h_{NE_v}^{(k)}))$. After generating $h_G^{(rep)}$, this is feed-forwarded into 2-layer MLP with ReLU activation function and softmax function to obtain a probability vector. The number of batch size and the dimension of hidden layer are both given by 64. Batch normalization \citep{Ioffe:2015:BNA:3045118.3045167} is applied after every hidden layer and dropout \citep{JMLR:v15:srivastava14a} ratio for the final prediction layer is given by $0.5$ \citep{xu2018how}. We used Adam optimizer \citep{DBLP:journals/corr/KingmaB14} with its learning rate $10^{-4}$ for toy-example plots in Section \ref{sec3.2:toyexp} and $10^{-2}$ for benchmark evaluations in Section \ref{sec3.3:graphbenchmark}, with exponential decay $0.99$. Cross entropy is used as a loss function. The detailed structure is illustrated in Fig. \ref{fig:model}.

\begin{figure}[htb]
    \centering
    \includegraphics[width=0.99\linewidth]{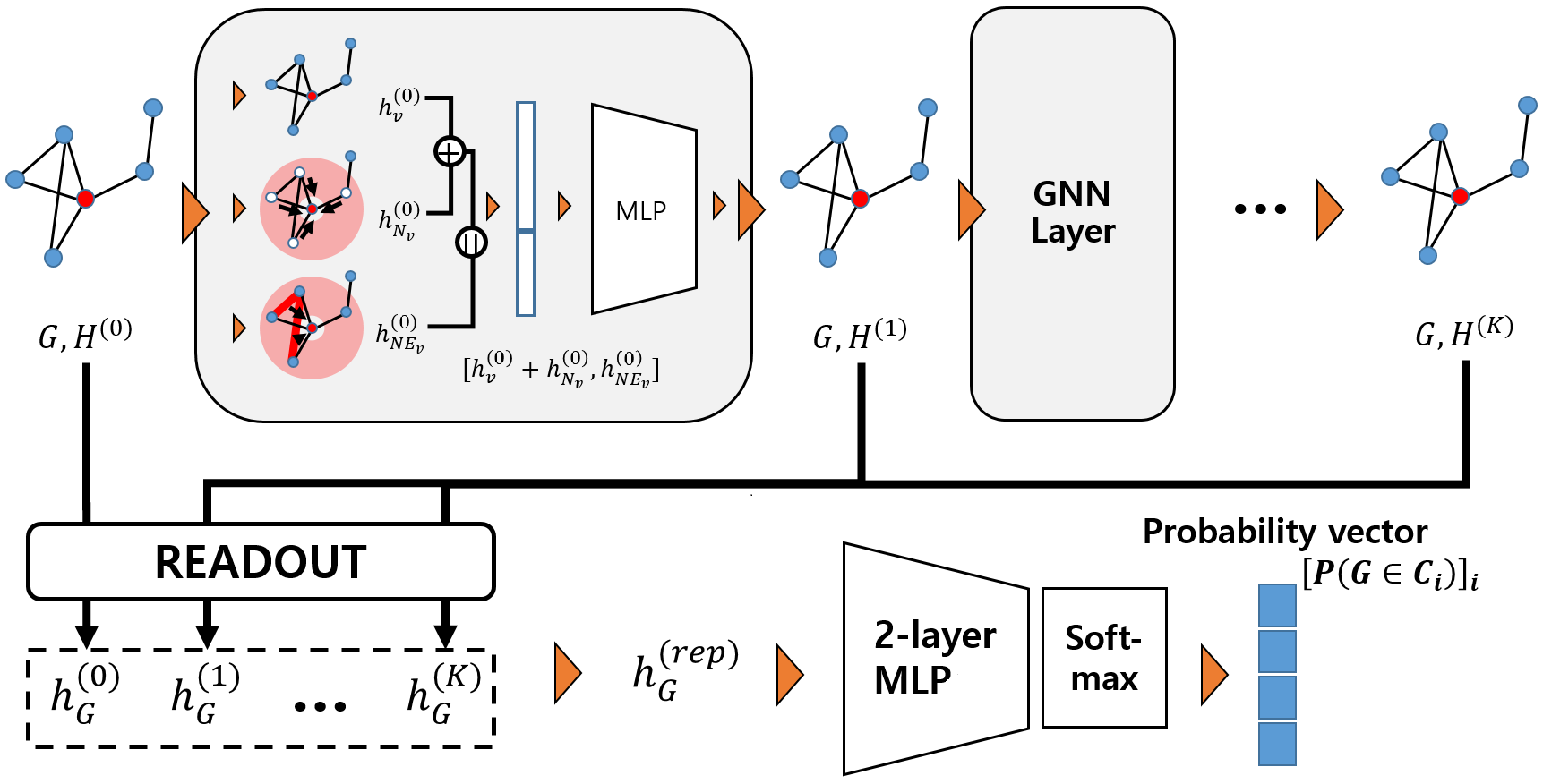}
    \caption{GNN model in our experiment. GIN-0 and NEAR are combined and aggregate neighborhoods' information recursively. READOUT extracts the graph representation vector $h_G^{(k)}$ of each graph $G, H^{(k)}$ and combines it into $h_G^{(rep)}$. This vector is put into 2-layer multilayer perceptron (MLP) classifier with a softmax activation and the model returns the probability vector.}
    \label{fig:model}
\end{figure}

\subsection{Toy example}
\label{sec3.2:toyexp}
In this experiment, we aim to show that NEAR can encode relations between the nodes in the neighborhoods. Comparing it with a plain GIN model and its variants, we can deduce that considering local structures to GNN layers is strongly required. We perform two graph classification tasks with the family of graphs introduced in Section \ref{sec2.2:toyex}, which was proven to be indistinguishable by GIN. Firstly, we generate 5000 artificial graphs that satisfy lemma \ref{lemma0}. Each graph has its node labels based on lemma \ref{lemma0}. We labeled the graphs into 5 balanced classes upon the clustering coefficient (ARTFCC) and the length of the longest simple cycle in the cycle basis (ARTFCY). 

We performed two graph classification tasks for the graphs above with a plain GIN model (GIN-0 in \cite{xu2018how} with sum AGGREGATOR) and 4 simple variants of NEAR with GIN-0: GIN-0 with NEAR-c/e/m/h. We report the loss curves, accuracies, and the entropy of predicted graph labels for fixed train/validation splits on ARTFCC/ARTFCY datasets in Fig. \ref{toyex31cc} and Fig. \ref{toyex31cy}. An entropy of predicted labels is reported to emphasize the fact that GIN with sum READOUT and ReLU activation maps almost all of graphs in the validation set to the same label.

\begin{figure*}[htb]
  \centering
  \includegraphics[width=0.485\linewidth]{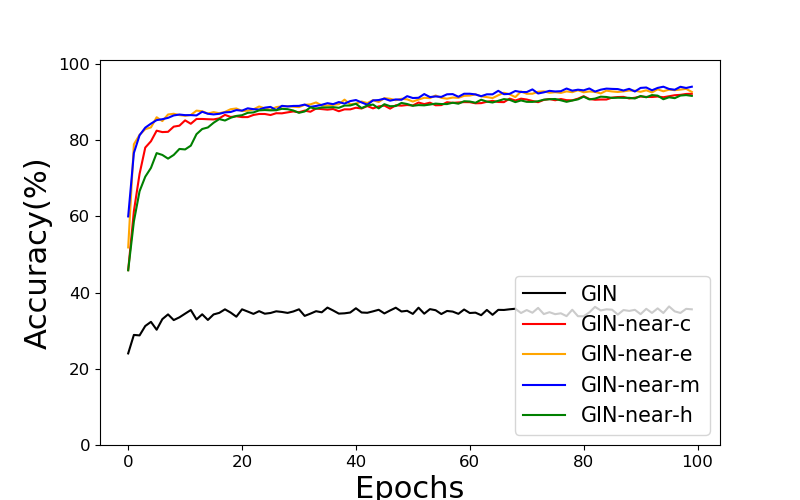}
  \includegraphics[width=0.485\linewidth]{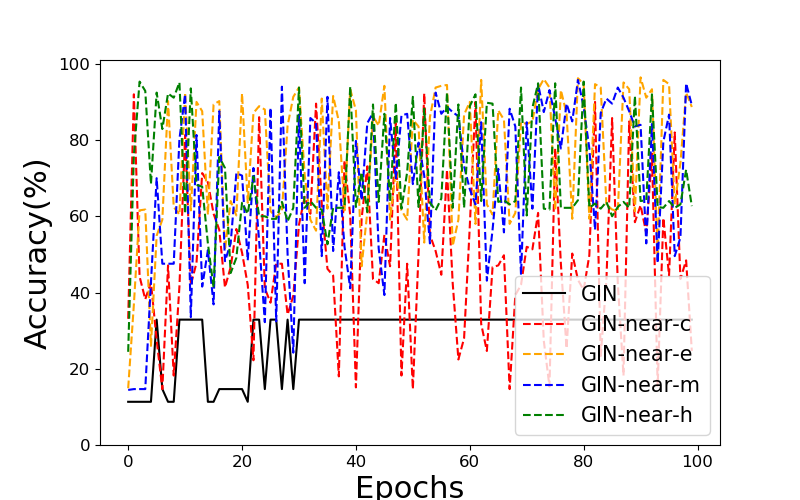}
  \\
  \includegraphics[width=0.485\linewidth]{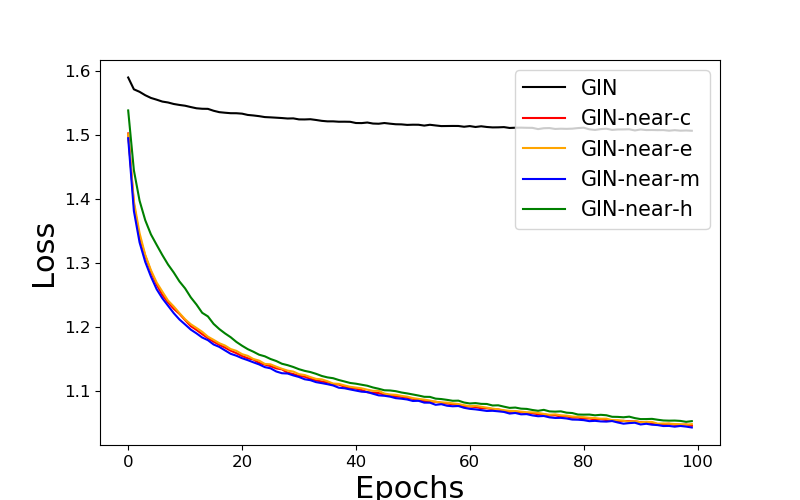}
  \includegraphics[width=0.485\linewidth]{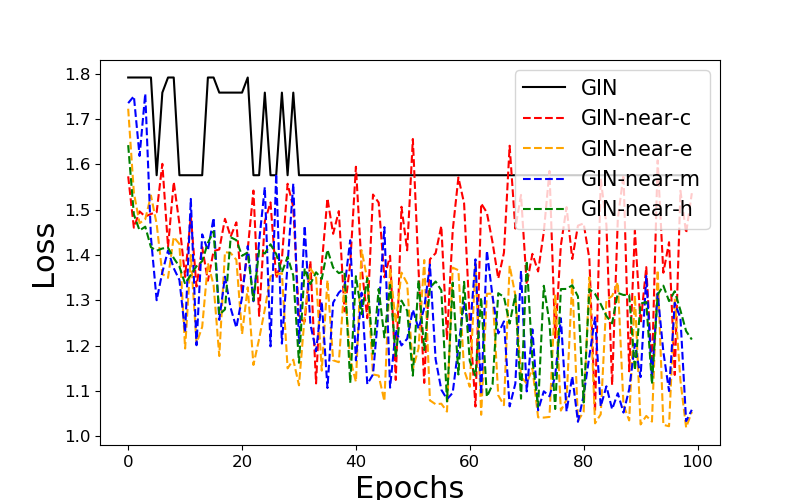}
  \\
  \includegraphics[width=0.485\linewidth]{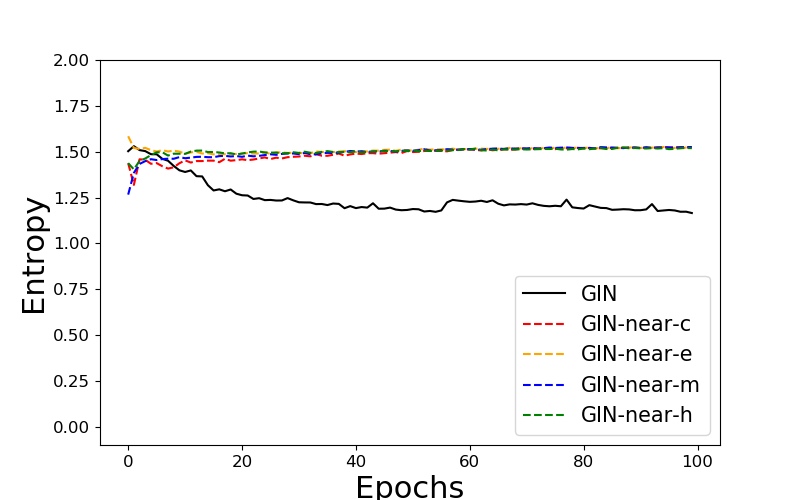}
  \includegraphics[width=0.485\linewidth]{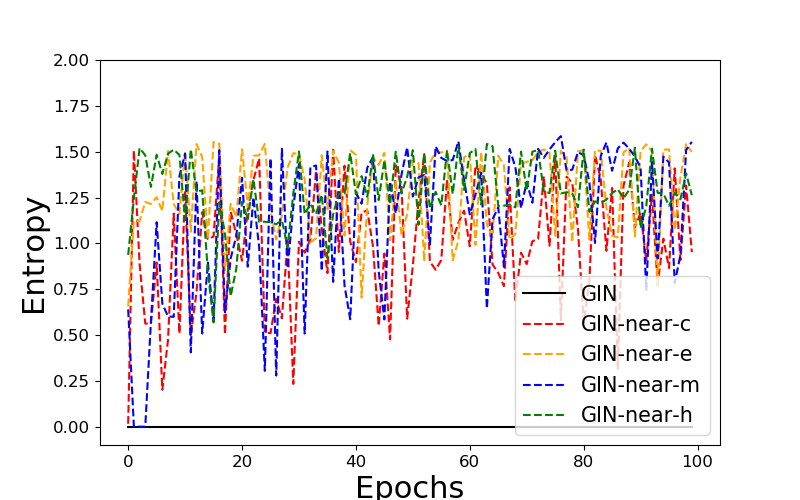}
  \\
  
  \caption{Accuracy, loss, and entropy of the predicted labels for training set (on the first column) and validation set (on the second column) of ARTFCC. GIN-0 is marked as black line, and GIN-NEAR-c/e/m/h are marked as dashed red/orange/blue/green lines.}
  \label{toyex31cc}
\end{figure*}

\begin{figure*}[htb]
  \centering
  \includegraphics[width=0.485\linewidth]{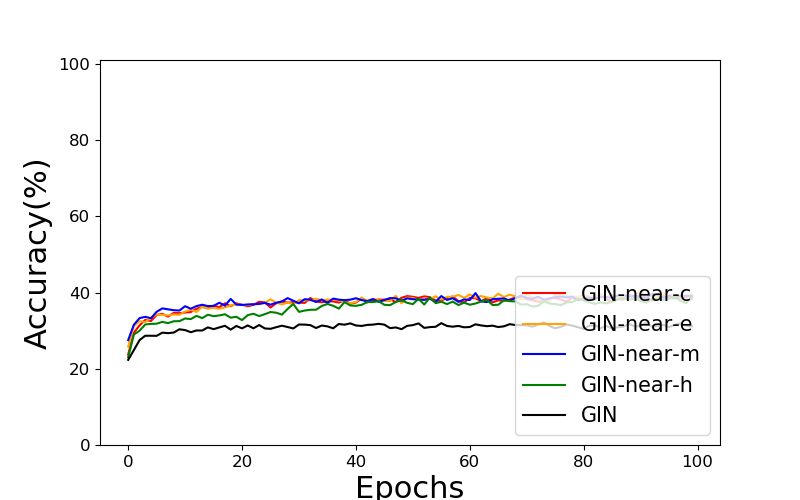}
  \includegraphics[width=0.485\linewidth]{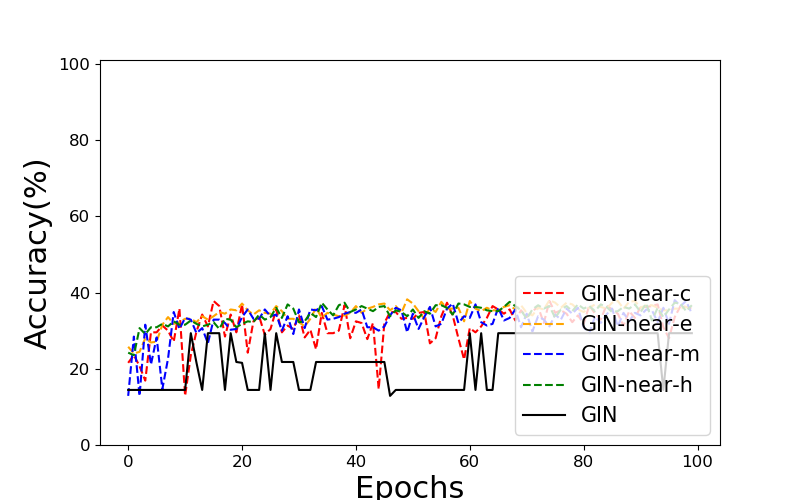}
  \\
  \includegraphics[width=0.485\linewidth]{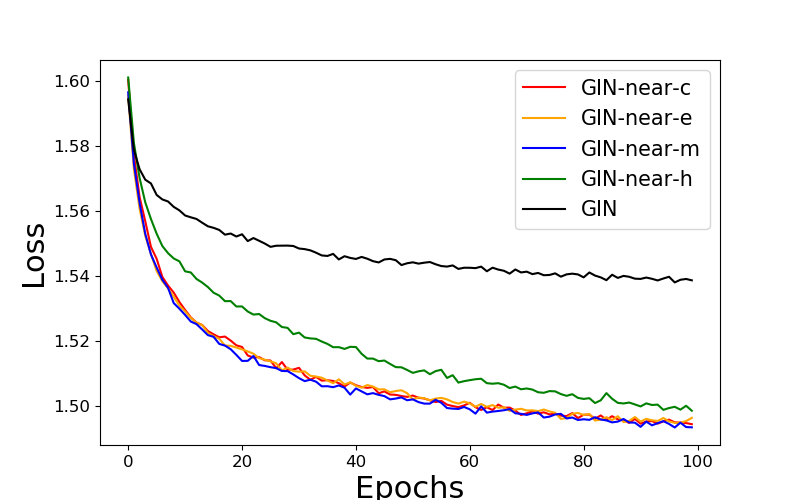}
  \includegraphics[width=0.485\linewidth]{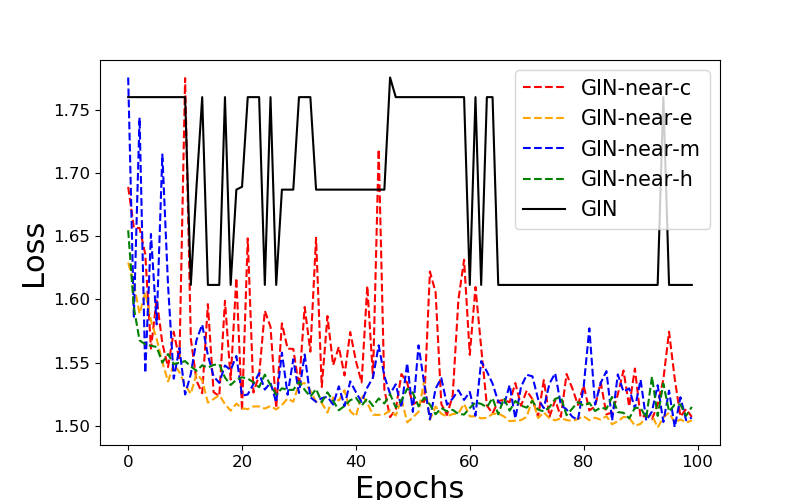}
  \\
  \includegraphics[width=0.485\linewidth]{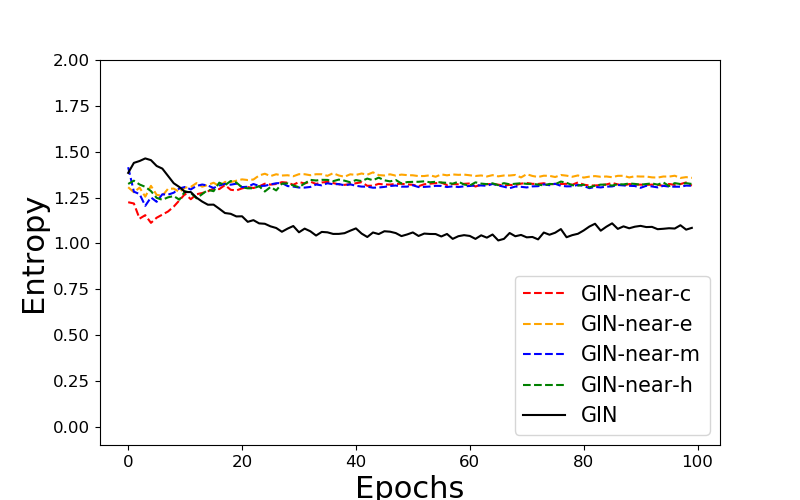}
  \includegraphics[width=0.485\linewidth]{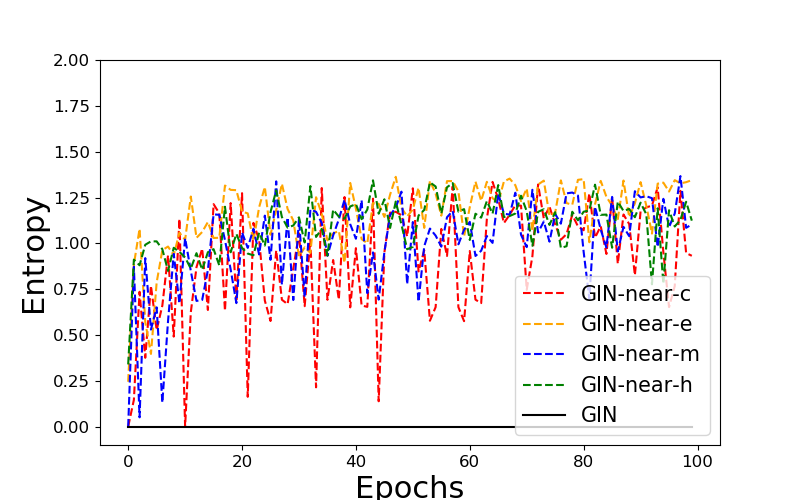}
  \\
  
  \caption{Accuracy, loss, and entropy of the predicted labels for training set (on the first column) and validation set (on the second column) of ARTFCY. GIN-0 is marked as black line, and GIN-NEAR-c/e/m/h are marked as dashed red/orange/blue/green lines.}
  \label{toyex31cy}
\end{figure*}

From the results on the toy examples with our GIN variants in Fig. \ref{toyex31cc} and Fig. \ref{toyex31cy}, we can observe that the training loss of GIN decreases slowly compared with the proposed NEAR variants, while other proposed NEAR variants were trained well. Moreover, the validation accuracy and validation loss of the plain GIN were not improved sufficiently well, whereas those of all the other NEAR variants got improved well. This implies that the plain GIN model predicted graphs only with $|V|$, which makes a step-function like behavior for the validation loss and validation accuracies. Besides, the entropy of predicted results on the validation set by GIN remains low, which shows that GIN predicts most of graphs in the validation set to be in the same label depending only on $|V|$. These results empirically show Theorem \ref{thm1} with sum READOUT function and ReLU activation. 

The graphs in these tasks have the same neighborhood sets with different local structures. Therefore, our proposed models also have the ability to catch differences between various graph structures. 

\subsection{Graph classification tasks}
\label{sec3.3:graphbenchmark}
Next, we perform graph classification tasks on eight real-world benchmark graph datasets, one synthetic benchmark dataset, and our two toy-example datasets. Four variants of NEAR (NEAR-c/e/m/h) are used as our proposed models in this experiment. Results of the recent algorithms based on GNNs and graph kernels are compared with ours. We executed all GNN-based models in our experiments.

\subsubsection{Datasets and features}
We use 6 chemistry/bioinformatics datasets (COX2, MUTAG, PTCMR, PROTEINS, NCI1, FRANKENSTEIN), 2 social-network datasets (IMDB-BINARY, IMDB-MULTI), and 1 synthetic dataset TRIANGLES as benchmark datasets. TRIANGLES is a dataset proposed by \cite{knyazev2019understanding} which is related with a number of triangles in the graphs. Discrete node labels were encoded into one-hot vectors, and continuous node attributes were used without preprocessing. If there are no node labels or attributes, we generated constant dummy labels or attributes for all nodes. Degree was one-hot encoded to a 500-dimensional unit vector and inserted as an additional node label. While preprocessing the degree, a node with degree 500 and a node with degree larger than 500 will mapped to the same one-hot vector.

Table \ref{dataset_info} shows descriptions and statistics for 9 graph benchmark datasets and 2 proposed datasets which are used in our experiments. The category of each dataset, number of graphs and classes, average number of nodes and edges are provided. If node features exist in the dataset, we use '+', otherwise '-'. If continuous node attributes exist in the dataset, their dimension is noted in parenthesis.

\begin{table}[htb]
  \caption{Detailed Data Statistics and Information}
  \label{dataset_info}
  \centering
  \begin{tabular}{llllrrrrl}
    \toprule
    \multicolumn{2}{c}{Dataset} &\multicolumn{2}{c}{Features} &\multicolumn{4}{c}{Statistics} \\
    \cmidrule(r){1-2}
    \cmidrule(r){3-4}
    \cmidrule(r){5-8}
    & Cate & \multicolumn{2}{c}{Node} & \multicolumn{2}{c}{Num. of} & \multicolumn{2}{c}{Average Num. of} \\
    Name & -gory & Label & Attr. & Graphs & Classes & Nodes & Edges \\
    \midrule
    COX2 & Bio & + & +(3) & 467 & 2 & 41.22 & 43.45 \\
    PROT & Bio & + & +(1) & 1113 & 2 & 39.06 & 72.82 \\
    FRANK & Bio & + & +(780) & 4337 & 2 & 16.90 & 17.88 \\
    MUTAG & Bio & + & $-$ & 188 & 2 & 17,93 & 19.79 \\
    PTC-MR & Bio & + & $-$ & 344 & 2 & 14.29 & 14.69 \\
    NCI1 & Bio & + & $-$ & 4110 & 2 & 29.87 & 32.30 \\
    IMDB-B & Soc & $-$ & $-$ & 1000 & 2 & 19.77 & 96.53 \\
    IMDB-M & Soc & $-$ & $-$ & 1500 & 3 & 13.00 & 65.94  \\
    TRIANG & Syn & $-$ & +(1) & 45000 & 10 & 20.85 & 32.74  \\
    \midrule
    ARTFCC & Syn & + & $-$ & 5000 & 5 & 26.78 & 33.48 \\
    ARTFCY & Syn & + & $-$ & 5000 & 5 & 26.78 & 33.48 \\    
    \bottomrule
  \end{tabular}
\end{table}

\subsubsection{Model settings}
We selected the basline algorithms as following: Graph Isomorphism Network (GIN) \citep{xu2018how}, Deep Graph Convolutional Neural Networks (DGCNN) \citep{dgcnn}, PATCHY-SAN (PSCN) \citep{patchysan}, and Random Walk Graph Neural Network (RWNN) \citep{DBLP:conf/nips/NikolentzosV20} were used as GNN-based baseline algorithms, and Return probability-based Graph Kernel (RGK) \citep{retgk}, Weisfeiler-Lehman subtree kernel (WL) \citep{shervashidze2011weisfeiler}, and Graph Neural Tangent Kernel (GNTK) \citep{DBLP:conf/nips/DuHSPWX19} are used as kernel-based baseline algorithms.

We trained 100 epochs for GIN, PSCN and NEAR variants with learning rate $10^{-2}$. PSCN with $k=10$ was used in our experiment. We used the model setting in Section 3.1 on the baseline GIN-0 and NEAR variants. Following the basic settings that are suggested in the official PyTorch implementation of DGCNN \citep{dgcnn}, we executed 500 epochs with learning rate $10^{-4}$ if the best parameter setting is not specified. For the kernel-based baselines, we tuned $C\in \{10^{-3},10^{-2},10^{-1},1,10^{1},10^{2},10^{3}\}$ in SVM classifiers. The number of iterations for WL-subtree was set to 4.

We followed the hyperparameters and detailed learning process in an official PyTorch implementation of GNTK and RWNN. \citep{DBLP:conf/nips/DuHSPWX19, DBLP:conf/nips/NikolentzosV20} For the datasets that are not mentioned in GNTK paper, we followed the setting of IMDB-BINARY, with the number of layers $\in\{2, 4\}$. For the datasets that are not mentioned in RWNN paper \citep{DBLP:conf/nips/NikolentzosV20}, we followed the hyperparemeter setting upon the category (social/bio) of the benchmark dataset. We also ran the experiment on synthetic datasets (TRIANGLES, ARTFCC, ARTFCY) based on the experiment setting for social network dataset of \citep{DBLP:conf/nips/NikolentzosV20}. We reported the results of $p$-step RWNN variants ($p$-step RWNN) with random walk parameter $p=1,2,3$.

We performed 10 times of 10-fold cross-validation. For TRIANGLE dataset, we chose hyperparmeters based on the fixed train/validation/test split (First 30000 graphs for training, next 5000 graphs for validation, last 10000 graphs for testing) as designed in \cite{knyazev2019understanding}.

\subsubsection{Results}
We reported the average of test accuracies and their standard deviations with our proposed algorithms in Table \ref{sample-table-bio} and Table \ref{sample-table-soc-artf}. NEAR, GIN, DGCNN, PSCN, RWNN variants, GNTK, and WL were executed for all datasets with the same splits, and RetGK was executed on TRIANGLES and our two toy example datasets. Available results of RetGK were reported directly. GNTK failed to generate a gram matrix for TRIANGLES dataset, so the test accuracy was not available and this is denoted as out of memory (OOM) in Table \ref{sample-table-soc-artf}. We highlighted the best results for each benchmark dataset in Table \ref{sample-table-bio} and Table \ref{sample-table-soc-artf}.

\begin{table*}[htb]
  \caption{Graph classification results for 6 bio-informatics benchmark datasets}
  \label{sample-table-bio}
  \centering
  \begin{tabular}{lllllll}
    \toprule
    & COX2 & PROT & FRANK & MUTAG & PTCMR & NCI1 \\
    \midrule
    WL & 81.3 & 75.1 & 74.2 & 85.8 & 61.8 & 84.3 \\
    RGK & 81.4 & \textbf{78.0} & \textbf{76.7} & \textbf{90.3} & 62.5 & \textbf{84.5} \\ 
    GNTK & \textbf{82.3} & 74.1 & 62.1 & 87.6 & \textbf{62.7} & 83.3 \\ 
    \midrule
    DGCNN & 78.2(6.8) & 72.9(4.4) & 68.7(2.1) & 83.8(7.7) & 56.3(8.0) & 71.8(2.6) \\
    PSCN & 74.0(6.5) & 67.3(4.2) & 57.7(2.1) & 82.7(7.1) & 59.4(9.6) & 70.0(2.3) \\
    1-step RWNN & 81.9(5.4) & 74.0(4.5) & 63.1(2.6) & 83.7(8.1) & 57.8(8.4) & 67.1(2.7) \\ 
    2-step RWNN & 81.4(5.9) & 74.6(4.3) & 68.1(2.6) & 87.0(6.7) & 57.9(8.5) & 71.3(2.4) \\
    3-step RWNN & 82.0(5.9) & 74.4(4.3) & 67.9(2.3) & 87.1(7.6) & 58.0(9.5) & 74.6(2.3) \\
    \midrule
    GIN-0 & 80.3(6.9) & 73.8(3.9) & 68.6(2.3) & 84.2(7.9) & 56.6(7.9) & 80.7(2.0) \\
    NEAR-c & 81.2(6.8) & 74.0(4.0) & 70.2(2.0) & 85.0(8.5) & 56.9(8.0) & 81.0(1.9) \\
    NEAR-e & 80.8(6.4) & 74.7(4.4) &  70.6(2.1) & 85.6(8.3) & 57.8(8.0) & 81.0(2.0) \\
    NEAR-m & 80.9(6.6) & 74.3(4.0) &  70.8(2.2) & 85.1(8.3) & 57.6(8.1) & 80.8(1.8) \\ 
    NEAR-h & 82.0(6.1) & 75.4(4.1) &  70.2(2.4) & 85.7(8.6) & 57.3(9.7) & 80.9(1.9) \\
    \bottomrule
  \end{tabular}
\end{table*}

\begin{table*}[htb]
  \caption{Graph classification results for 2 social-network benchmark datasets and 3 synthetic benchmark datasets}
  \label{sample-table-soc-artf}
  \centering
  \begin{tabular}{lllllll}
    \toprule
    & IMDB-B & IMDB-M & TRIANG & ARTFCC & ARTFCY \\
    \midrule
    WL & 71.5 & \textbf{51.9} & 45.3 & 32.9 & 28.3 \\
    RGK & 72.3 & 48.7 & 59.1 & \textbf{99.0} & 47.0 \\ 
    GNTK & 73.6 & 50.9 & OOM & 64.1 & \textbf{69.7} \\ 
    \midrule
    DGCNN & 67.8(5.1) & 44.0(4.8) & 48.5(3.0) & 42.0(2.4) & 35.1(1.9) \\
    PSCN & 65.0(5.4) & 44.9(3.5) & 59.2(3.6) & 98.0(0.6) & 38.3(2.1) \\
    1-step RWNN & 70.6(4.7) & 47.5(4.3) & \textbf{99.8(0.5)} & 35.4(2.0) & 35.5(1.9) \\ 
    2-step RWNN & 71.0(5.1) & 48.0(4.1) & 97.5(2.3) & 35.3(1.9) & 35.4(2.0) \\
    3-step RWNN & 70.9(4.7) & 47.8(4.1) & 93.8(4.3) & 35.4(2.3) & 35.4(2.0) \\
    \midrule
    GIN-0 & 72.8(4.3) & 51.0(3.8) & 75.2(2.6) & 32.2(2.9) & 27.1(3.4) \\
    NEAR-c & 72.9(4.6) & 51.1(4.3) & 78.7(2.9) & 97.6(0.9) & 39.7(2.0) \\
    NEAR-e & 72.9(4.3) & 51.2(4.0) & 78.2(4.2) & 97.9(0.6) & 40.4(2.1) \\
    NEAR-m & \textbf{73.7(4.3)} & 50.8(3.8) & 79.7(3.2) & 97.9(0.7) & 40.2(2.1) \\
    NEAR-h & 73.5(4.4) & 50.8(3.9) & 80.5(3.5) & 97.8(0.7) & 40.2(2.4) \\
    \bottomrule
  \end{tabular}
\end{table*}

Although most of the state-of-the-art results were attained from the kernel-based methods, our proposed model NEAR achieved the compatible results from the GNN-based baselines, and slightly improved the benchmark results of previous GIN-0 model. We remark that due to the small size of the dataset, a high variance is observed among the test accuracies. As a consequence,  it is hard to statistically claim that the result on the real-world benchmark datasets is significantly improved by NEAR. However, a direct comparison with GIN-0 and our NEAR variants shows that our neighborhood-relation encoding indeed improved the model capacity of GIN-0. Our combined model with GIN-0 and NEAR improved the results for comparable datasets, especially for synthetic datasets which have a high demand for catching a detailed difference in 1-hop neighborhood. Thus, we can conclude that the edge-aggregating framework can be orthogonally combined with GIN and increase the capacity of the model.

\section{Conclusion}
We proposed NEAR, a new GNN framework that aggregates edges in the neighborhood and enables to encode the local structures to hidden vectors. We constructed a family of graphs with the same neighborhoods and distribution of labels but with different local structures. By using the proposed edge-aggregating framework with GIN models, we showed that NEAR has the ability to encode local structures and we obtained exemplary results for several graph classification tasks. Our proposed algorithm NEAR shows a better model capacity to deal with both local structures of graphs and node labels/attributes. Possible future work would be finding a more powerful and computationally efficient function $g(h_u, h_z)$ and a neighborhood edge multiset function $\phi$ that can represent connections in a neighborhood. Additionally, encoding edge labels/attributes with NEAR and GNN layers would be fruitful for more complex graph classification tasks and graph embedding.

\section*{Acknowledgement}
This work was supported by the Basic Science Research Program through the National Research Foundation of Korea NRF-2017R1E1A1A03070105 and NRF-2019R1A5A1028324 and by Institute for Information \& communications Technology Promotion(IITP) grant funded by the Korea government(MSIP) (2019-0-01906, Artificial Intelligence Graduate School Program(POSTECH)) and by the ITRC (Information Technology Research Center) support program (IITP-2018-0-01441).

\bibliographystyle{plain}
\bibliography{ijcai20_near}
\end{document}